\documentclass[12pt,nonatbib]{anon}
\usepackage{times}

\usepackage{amsmath}
\usepackage{amssymb}
\usepackage{mathtools}
\usepackage{bbm} 
\usepackage{nicefrac}
\usepackage{enumitem}

\usepackage[normalem]{ulem}

\renewcommand{\nicefrac}[2]{#1/#2}

\usepackage{parskip}

\usepackage{xspace}

\newcommand{\CSC}{\textsc{CSC}\xspace}
\newcommand{\alg}{\textsc{PIWO-IX}\xspace}

\newcommand{\real}{\mathbb{R}}

\newcommand{\KL}[2]{\mathrm{KL}\pa{#1\middle\|#2}}

\newcommand{\OO}{\mathcal{O}}
\newcommand{\dd}{\mathrm{d}}

\newcommand{\II}[1]{\mathbb{I}_{\left\{#1\right\}}}
\newcommand{\PP}[1]{\mathbb{P}\left[#1\right]}

\newcommand{\EE}[1]{\mathbb{E}\left[#1\right]}

\newcommand{\EEcc}[2]{\mathbb{E}\left[\left.#1\right|#2\right]}

\def\argmin{\mathop{\mbox{ arg\,min}}}
\newcommand{\ra}{\rightarrow}

\newcommand{\ev}[1]{\left\{#1\right\}}
\newcommand{\pa}[1]{\left(#1\right)}
\newcommand{\bpa}[1]{\bigl(#1\bigr)}

\newcommand{\wh}{\widehat}
\newcommand{\wt}{\widetilde}

\newcommand{\hV}{\wh{v}}
\newcommand{\tV}{\wt{v}}
\newcommand{\hpi}{\wh{\pi}}
\newcommand{\tr}{\wt{r}}

\newcommand{\hQ}{\wh{Q}}

\newcommand{\transpose}{^\mathsf{\scriptscriptstyle T}}

\definecolor{PalePurp}{rgb}{0.66,0.57,0.66}

\newcommand{\A}{{\mathcal{A}}}
\newcommand{\X}{{\mathcal{X}}}

\newcommand{\D}{\mathcal{D}}

\newcommand{\n}{n}
\newcommand{\sumt}{\sum^{\n}_{t=1}}
\newcommand{\suma}{\sum_{a}}

\newcommand{\One}[1]{\mathbbm{1}\{ #1 \}}

\DeclarePairedDelimiterX\ip[2]{\langle}{\rangle}{#1,#2}

\let\P\undefined
\DeclarePairedDelimiterXPP\P[1]{\mathbb{P}}(){}{
    
    #1
}

\DeclarePairedDelimiterXPP\E[1]{\mathbb{E}}[]{}{
    
    #1
}

\DeclarePairedDelimiterXPP\Es[2]{\mathbb{E}_{#1}}[]{}{
    
    #2
}

\DeclarePairedDelimiterXPP\Ipt[2]{{#1}\transpose}(){}{#2}
\DeclarePairedDelimiterXPP\Iptr[2]{}(){\transpose{#2}}{#1}

\newcommand{\Reg}{\mathfrak{R}}

\newcommand{\Reals}{\mathbb{R}}

\title{Importance-Weighted Offline Learning Done Right}

\author[Gabbianelli, Neu \& Papini]{%
	\Name{Germano Gabbianelli} \Email{germano.gabbianelli@upf.edu}\\
	\addr Universitat Pompeu Fabra, Barcelona, Spain
	\AND
	\Name{Gergely Neu} \Email{gergely.neu@gmail.com} \\
	\addr Universitat Pompeu Fabra, Barcelona, Spain
	\AND
	\Name{Matteo Papini} \Email{matteo.papini@polimi.it} \\
	\addr Politecnico di Milano, Milan, Italy
}

\DeclareMathOperator*{\argmax}{argmax}

\newcommand{\rht}{\widehat{r}_t(\pi)}
\newcommand{\rtt}{\widetilde{r}_t(\pi)}

\newcommand{\tpi}{{\widetilde{\pi}}}

\begin{document}
    \maketitle

\begin{abstract}
We study the problem of offline policy optimization in stochastic contextual bandit problems, where the goal is to 
learn a near-optimal policy based on a dataset of decision data collected by a suboptimal behavior policy. 
Rather than making any structural assumptions on the reward function, we assume access to a given policy class and aim 
to compete with the best comparator policy within this class. In this setting, a standard approach is to compute 
importance-weighted estimators of the value of each policy, and select a policy that minimizes the estimated value up 
to a ``pessimistic'' adjustment subtracted from the estimates to reduce their random fluctuations. In this paper, we 
show that a simple alternative approach based on the ``implicit exploration'' estimator of \citet{Neu2015} yields 
performance guarantees that are superior in nearly all possible terms to all previous results. Most notably, we remove 
an extremely restrictive ``uniform coverage'' assumption made in all previous works. These improvements are made 
possible by the observation that the upper and lower tails importance-weighted estimators behave very 
differently from each other, and their careful control can massively improve on previous results that were all based on 
symmetric two-sided concentration inequalities. 
We also extend our results to infinite policy classes in a PAC-Bayesian fashion, and showcase the robustness of our 
algorithm to the choice of hyper-parameters by means of numerical simulations.

    \end{abstract}

\section{Introduction}
Offline Policy Optimization (OPO) is the problem of learning a near-optimal policy based on a dataset of historical 
observations. This problem is of outstanding importance in real-world applications where 
experimenting directly with the environment is costly, but otherwise large volumes of offline data is available to 
learn 
from. Such settings include problems in healthcare \citep{Murphy2003,Kim2011,Bertsimas2016,TM17},
advertising \citep{Bottou2013,Farias2019}, or recommender systems \citep{LCLW11,Schnabel2016}.

A popular approach for this setting is \emph{importance-weighted offline learning}, where one optimizes an unbiased 
estimate of the expected reward, obtained through an appropriately reweighted average of the rewards in the 
dataset \citep{LCLW11,Bottou2013}. To deal with unstable nature of these estimators, the influential work of 
\citet{Swaminathan2015} proposed an approach called ``counterfactual risk minimization'', which consists of adding 
a regularization term to the optimization problem to down the fluctuations, thus preventing the optimizer to 
overfit to random noise. Their work has inspired a number of follow-ups that either refined the regularization 
terms to yield better theoretical guarantees \citep{Jin2023,Wang2023}, or developed practical methods with 
improved empirical performance in large-scale problems \citet{london2019bayesian,Sakhi2023}. In this paper, we 
contribute to this line of work by studying a simple and robust variant of the standard importance-weighted 
reward estimators used in past work, and showing tight theoretical performance guarantees for it.

Our main contribution is showing that the so-called \emph{implicit exploration} (IX) estimator (originally proposed 
by \citealp{KNVM14} and \citealp{Neu2015} in the context of online learning) achieves a massive variance-reducing 
effect in our offline learning setting, and using this observation to derive performance guarantees that are both 
significantly tighter and easier to interpret than all previous results in the literature. In particular, we 
formally show that the regularization effect built into the IX estimator is strong enough so that no further 
regularizer is required to stabilize the performance of policy optimization. This result is perhaps 
surprising for the reader familiar with past work on the subject, especially since several of these works made use 
of IX-like variance reduced estimators without managing to drop the additional regularization. The key observation 
that allows us to prove our main results is that the tails of importance-weighted estimators are 
\emph{asymmetric}, which allows us to tightly control the two tails separately via specialized concentration 
inequalities. This is to be contrasted with previous results that all rely on symmetric confidence intervals 
that turn out to be needlessly conservative. This new perspective not only allows us to obtain better 
results but also to simplify the analysis: both of the concentration inequalities we use for the two tails can be 
derived using elementary techniques in a matter of a few lines\footnote{In fact, both results are readily 
available in the literature: one is the main result of \citet{Neu2015} regarding the upper tail of the IX 
estimator, and another is stated as an exercise in \citet{BLM13}.}.

More concretely, our main result is a regret bound that scales with the degree of ``overlap'' between the 
comparator policy and the behavior policy, demonstrating better scaling against policies that are covered better 
by the observed data. Unlike virtually all previous work, our guarantees do not require the unrealistic 
condition that action-sampling probabilities be bounded away from zero for all contexts. Our algorithm can be 
implemented efficiently using a single call to a cost-sensitive classification oracle, thus effectively reducing 
the offline policy optimization problem to a standard supervised learning task (which feature is in high 
regard thanks to the influential works of \citealp{Langford2007,Dudik2011,Agarwal2014} in the broader area of 
contextual bandit learning). For simplicity of exposition, we prove our main result for finite policy classes and 
show that the regret scales logarithmically with the size of the class. We also provide some extensions to the
simple algorithm achieving these results, namely a version that trades oracle-efficiency for a better scaling 
with the quantity measuring the mismatch between the target and behavior policies, and a ``PAC-Bayesian'' 
variant that can make use of prior information on the problem and also works for infinite policy classes. This 
extends the recent works of \citet{london2019bayesian,Sakhi2023,Flynn2023} by providing better generalization 
bounds and introducing a new family of PAC-Bayesian regret bounds that apparently have not existed so far in the 
literature. We also illustrate our theoretical findings with a set of experiments conducted on real data, and 
empirically verify the robustness of our method as compared to some natural baselines.
    
It is worth mentioning a parallel line of work on contextual bandits that starts from the assumption that 
the reward function belongs to a known function class, and thus a near-optimal policy can be learned by identifying 
the true reward function within the class up to sufficient accuracy. This perspective has been adopted by 
\citet{JYW21} (as well as a sequence of follow-up works on offline reinforcement learning) who considered function 
classes that are linear in some  low-dimensional features of the  context-action pairs. These works provide simple 
algorithms with strong theoretical performance guarantees, but they are all limited by the strong assumptions that need 
to be made about the reward function (and it is unclear how sensitive they are to model misspecification). In contrast, 
the setting we consider assumes access to a policy class and allows the development of algorithms that perform nearly 
as 
well as the best policy within the class \emph{without} requiring that the rewards have a simple parametric form. This 
setting comes with its own set of trade-offs: the statistical complexity of learning in this setting depends on the 
complexity of the policy class, and hard problems will evidently require large classes of policies to accommodate 
best-in-class policies with satisfying performance. Our results in this paper highlight some further open questions in 
this setting regarding computational-statistical trade-offs---the discussion of which we relegate to 
Section~\ref{sec:discussion}.

    \section{Preliminaries}
    We study the problem of offline learning in stochastic contextual bandits. 
    A contextual bandit problem instance is  defined by the tuple $(\X,\A,\nu,p)$,
    where $\X$ is a set of contexts, $\A$ is a set of actions with finite cardinality $K$, $\nu$ is an unknown 
    distribution over the context space representing the probability of encountering each context,
    and $p:\X\times\A\ra \real_{[0,1]}$ is a probability kernel mapping context-action pairs to rewards in the interval 
$[0,1]$. The mean reward     associated with context-action pair $x,a$ is denoted as $r(x,a)$.
    
    A policy $\pi: \X\ra\Delta_{\A}$ is defined as mapping from contexts to
    distributions over actions, with $\pi(a|x)$ denoting the probability of 
    selecting action $a \in \A$ in context $x\in\X$ when following policy $\pi$.
    The expected reward of a policy $\pi$ is called \emph{value} and is defined as 
    $v(\pi) = \EE{\sum_{a} \pi(a|X) r(X,a)}$.
    We are given access to a dataset $\D=(X_t,A_t,R_t)_{t=1}^n$, sampled
    by a fixed behavior policy $\mu: \X\ra\Delta_{\A}$ according to the following 
    protocol:
    \begin{itemize}[noitemsep]
     \item $X_t$ is a context drawn i.i.d.~from the unknown distribution $\nu$,
     \item $A_t$ is an action drawn from the behavior policy $\mu(\cdot|X_t)$
         independently from all random variables other than $X_t$,
     \item $R_t$ is a random reward drawn from $p(\cdot|X_t,A_t)$, assumed to lie almost surely in $[0,1]$,
         with mean given by the reward function $\EEcc{R_t}{X_t,A_t} = r(X_t,A_t)$.
    \end{itemize}
    For simplicity, we suppose that the behavior policy $\mu$ is fixed and known,
    and only note here that extension to adaptive behavior policies is straightforward.

    The goal is to use the available data to produce a policy $\tpi_n$ achieving
    the highest possible expected reward. The performance will be measured
    in terms of \emph{regret} (or \emph{excess risk}) with respect to a
    comparator policy $\pi^*$:
    \[
        \Reg_n(\pi^*) = v(\pi^*) - v(\tpi_n).
    \]
    We assume to have access to a policy class $\Pi \subseteq \ev{\pi: \X\ra\Delta_\A}$ and aim to provide regret 
    bounds against all policies within the class. For most of our contributions, we will work with finite policy 
    classes  and assume access to a computational oracle that can return optimal policies given an appropriately 
    defined input dataset, Precisely, the oracle takes as input a dataset $\ev{x_t, g_t}_{t=1}^n$ with 
    contexts $x_t\in\X$ and gains $g_t\in\Reals^\A$, and returns
    \[
        \CSC\bpa{\ev{x_t, g_t}_{t=1}^n} = \argmax_{\pi\in\Pi} \sumt  \sum_a \pi(a|x_t) {g_t(a)}.
    \]
    This definition is slightly more general than the one used in previous works such as \citet{Dudik2011}, in that it 
    allows $\Pi$ to include stochastic policies. The optimization problem solved by the oracle is easily seen to be 
    equivalent to the task of \emph{cost-sensitive classification} (CSC), when considering the rewards of each action 
    as     negative costs associated with misclassification errors. We will thus occasionally refer to the oracle as 
    a \emph{CSC oracle} (which also explains the notation used above). 
    We are interested in developing algorithms that 
access     the  oracle a small constant number     of times while providing formal performance guarantees on the 
quality 
of the     output policy.

    \section{Pessimistic importance-weighted offline learning in contextual bandits} \label{sec:PIWO}
    A natural approach for the offline learning setting we consider is
    to define an estimator $\hV_n(\pi)$ for the value function $v(\pi)$ of each
    policy $\pi$, and to return the policy $\hpi_n\in\Pi$ which maximizes it.
    The simplest possible estimator one can think of is the \emph{importance-weighted} (IW) value estimator 
    \citep{HT52} defined for each policy $\pi$ as
    \begin{equation}\label{eq:IW}
        \hV_n(\pi) = \frac 1n \sum_{t=1}^n \frac{\pi(A_t|X_t)}{\mu(A_t|X_t)} \cdot R_t.
    \end{equation}
    This estimator is unbiased in the sense that for any policy $\pi$ we have
    $\E{\hV_n(\pi)} = v(\pi)$, and the output policy $\hpi_n$ can be computed
    with a single call to the computational oracle, by setting the reward vectors as $g_t(a) =
    \nicefrac{\One{A_t = a}R_t}{\mu(A_t|X_t)}$.
    However, this estimator is notoriously heavy-tailed, and thus,
    with significant probability, may be very distant from its true value \citep{Ion08}.
    Specifically, actions that are sampled with very low probability by the behavior policy may
    falsely appear to yield huge rewards. This is especially detrimental in
    offline learning where fitting a policy to such sampling artifacts can lead to poor
    performance during deployment.

    In recent years, a range of ideas have been proposed to tame the adverse
    behavior of the IW estimator. The most widely adopted approach, first proposed by
    \cite{Swaminathan2015} (and later elaborated on in a variety of contexts by works like
    \citealp{london2019bayesian,JYW21,rashidinejad2021bridging,li2022pessimism,Jin2023}), involves 
    implementing a ``pessimistic'' adjustment to reduce random
    fluctuations. Concretely, this method involves calculating an
    adjustment term $B_n(\pi)$ which, when subtracted from the
    IW estimator $\hV_n(\pi)$, ensures that the result is always smaller than the true
    value $v(\pi)$ for any policy $\pi$. Subsequently, the best pessimistic
    policy in $\Pi$ is identified and returned, by maximizing the expression
    $\hV_n(\pi)-B_n(\pi)$. The adjustment $B_n(\pi)$ is typically computed using standard
    concentration inequalities like Bernstein's inequality (see, e.g., Sections 2.7 and 2.8 in \citealp{BLM13}), and 
    generally tends to grow larger as the policy $\pi$ deviates further from the behavior policy.
    It is then easy to show that the regret of this method with respect to any comparator
    $\pi^*$ can be bounded by $2B_n(\pi^*)$, via
    a straightforward calculation which we reproduce in the proof of our Theorem~\ref{thm:main}.

    This generic recipe for offline learning has been combined with the
    IW estimator defined above by \citet{Swaminathan2015}, \citet{Jin2023} and
    \citet{Wang2023}. This ``pessimistic importance-weighted offline learning''
    approach,
    which we abbreviate as \emph{PIWO learning}, has several downsides,
    depending on the choice of $B_n(\pi)$. First, as pointed out recently by
    \citealp{Wang2023}, $\hV_n - B_n$ may not be necessarily be of the form
    required by a practical optimization oracle. Even more
    concerningly, a conservatively chosen adjustment $B_n$ may not only result
    in loose theoretical guarantees, but also poor empirical performance.
    Indeed, notice that setting $B_n$ too large may overwhelm the data-dependent
    value estimates, thus resulting in a policy that effectively ignores the
    observed data from policies that are relatively poorly covered. In extreme cases, this approach may even favor 
    policies that have never been observed to yield any reward whatsoever over policies with positive estimated 
    reward but high estimated uncertainty.

    The cleanest results for this PIWO learning approach have been derived by \citet{Wang2023}, who used the 
    adjustment $B_n(\pi) = \beta \sum_{t=1}^n \sum_a\frac{\pi(a|X_t)}{\mu(a|X_t)}$. Their regret bounds are stated in 
    terms of the following quantity that measures the ``overlap'' between a given policy $\pi$
    and the behavior policy $\mu$:
    \begin{equation}\label{eq:coverage}
     C(\pi) = \EE{\sum_{a} \frac{\pi(a|X)}{\mu(a|X)}}.
    \end{equation}
    We will refer to $C(\pi)$ as the \emph{policy coverage ratio} between $\pi$
    and $\mu$. The coverage ratio can be seen as a notion of similarity between
    $\pi$ and $\mu$: it is minimized when the two policies are equal, and
    otherwise grows to infinity as the two policies drift apart. Assuming that
    the likelihood ratio between the two policies is uniformly upper-bounded as
    $\sup_{x,a} \frac{\pi(a|x)}{\mu(a|x)} \le \frac{1}{\alpha}$, \citet{Wang2023}
    obtain, for their oracle-efficient algorithm, a regret bound of the form 
    \begin{equation}\label{eq:PIWO-bound}
\Reg_n(\pi^*) = \OO\pa{C(\pi^*)\sqrt{\frac{\log\pa{|\Pi|/\delta}}{n}} + \frac{\log(|\Pi|/\delta)}{\alpha n}}.
    \end{equation}
    This bound has the appealing property that its leading term scales as $C(\pi^*)/\sqrt{n}$, thus guaranteeing good 
    performance when the comparator policy $\pi^*$ is well-covered by the behavior policy. The bound can be improved to 
    scale with $\sqrt{C(\pi^*)}$ instead of $C(\pi^*)$ if one has prior knowledge of the coverage ratio against the 
    target policy $\pi^*$. 
    On the negative side, the result effectively requires the strong \emph{uniform coverage} condition 
    $\inf_{x,a} \mu(a|x) \ge \alpha$ which ensures that all actions are sampled at least a constant $\alpha$
    fraction of times in the data set. This condition is typically not met in realistic applications for reasonable 
    values of $\alpha$, and in particular the bound becomes completely void of meaning if there exists one single 
    context $x$ where some action $a$ is selected with zero probability.
    
    The original algorithm by \cite{Swaminathan2015} suffered from the same issue. Recently, \cite{Jin2023} were able 
    to relax this uniform-coverage condition by developing a sophisticated concentration inequality that only requires 
    the third moment of the importance weights     $\sum_a \frac{\pi^*(a|X_t)}{\mu(a|X_t)}$ to be bounded. Eventually, 
    their bounds only apply to deterministic 
    policies that map each context $x$ to a single action $\pi^*(x)$, and depend on the quantity $\alpha^* 
    = \inf_{x} \mu(\pi^*(x)|x)$. Their most clearly stated result is Corollary~4.3, where they effectively show
    \[
     \Reg_n(\pi^*) = \OO\pa{\sqrt{\frac{\log\pa{|\Pi|T}}{\alpha^* n}}\cdot\pa{\log\pa{\frac 1\delta}}^{3/2}}.
    \]
    This bound still remains vacuous if there is one single context where $\mu(\pi^*(x)|x)$ is zero. A further downside 
    of their method pointed out by \citet{Wang2023} is that the proposed algorithm is not directly implementable with a 
    CSC oracle due to the form of the adjustment $B_n$ they use. In the following section, we will develop an 
    algorithm that eliminates all these limitations.

\section{Pessimism and Variance Reduction via Implicit Exploration} \label{sec:PIWO-IX}
Our main contribution is addressing the limitations of the PIWO learning framework in the previous section by studying 
a very simple adjustment to the standard IW estimator. Concretely, we adapt the so-called ``Implicit eXploration'' (IX) 
estimator of \cite{Neu2015} defined as
\begin{equation}\label{eq:IX}
 \tV_n(\pi) = \frac 1n \sum_{t=1}^n \frac{\pi(A_t|X_t)}{\mu(A_t|X_t) + \gamma} \cdot R_t,
\end{equation}
where $\gamma \ge 0$ is a hyperparameter of the estimator that we will sometimes refer to as the ``IX parameter''. 
This adjustment implicitly acts like mixing the behavior 
policy with a uniform exploration policy, thus reducing the random fluctuations of the IW estimator (and justifying the 
name ``implicit exploration''). The price of this stabilization effect is that the estimates are biased towards zero to 
an extent that can be controlled using the IX parameter $\gamma$. Indeed, as a simple calculation shows, the IX 
estimator satisfies $\EE{\tV_n(\pi)} = V(\pi) - \gamma C_\gamma(\pi)$, with the bias term $C_\gamma(\pi)$ given as
\begin{equation}\label{eq:coverage_gamma}
C_\gamma(\pi) = \EE{\suma \frac{\pi(a|X)}{\mu(a|X)+\gamma}\cdot r(X,a)}.
\end{equation}
Since the rewards are assumed to be non-negative, this bias can be interpreted as a \emph{pessimistic} adjustment 
to an otherwise unbiased estimator, and it is thus reasonable to expect it to have the same effect as the adjustments 
used in the general PIWO framework\footnote{In fact, the pessimistic bias of the IX estimators has been recently 
pointed out and utilized by \citet{GNP23} in the vaguely related context of online learning with off-policy 
feedback.}.

Note that $C_\gamma(\pi)$ is closely related to the policy coverage ratio $C(\pi)$ as defined in 
Equation~\eqref{eq:coverage}, up to the two differences that \emph{i)} it replaces $\mu(X,a)$ by $\mu(X,a) + \gamma$ 
in the denominator and \emph{(ii)} it is scaled with the rewards $r(X,a)$. Both of these adjustments make it strictly 
smaller than $C(\pi)$ as long as $\gamma > 0$, and notably it always remains bounded as $C_\gamma(\pi) \le 
\frac{1}{\gamma}$, no matter how small $\mu(a|x)$ gets. Furthermore, due to the scaling with the 
rewards, $C_\gamma(\pi)$ is small for policies with low expected reward, and in particular it equals zero for a policy 
with zero expected reward. In what follows, we will refer to $C_\gamma$ as the \emph{smoothed coverage 
ratio}.\footnote{We use this term in the sense of the Laplace 
smoothing of estimators, not to be confused with the smoothed analysis of algorithms~\citep{spielman2001smoothed} 
applied to contextual bandits by \cite{krishnamurthy2019contextual}.}

Our algorithm consists of simply selecting the policy that maximizes the IX value estimates:
\[
 \hpi_n = \arg\max_{\pi\in\Pi} \tV_n(\pi).
\]
We refer to this algorithm as \alg, standing for ``Pessimistic Importance-Weighted Offline learning with Implicit 
eXploration''. Note that \alg can be implemented via a single call to 
the CSC oracle with the gain vectors defined as $g_t(a)=\nicefrac{\One{A_t=a}R_t}{(\mu(A_t|X_t)+\gamma)}$. 
The following theorem 
states our main result regarding \alg.
\begin{theorem}\label{thm:main}
With probability at least $1-\delta$, the regret of \alg against any comparator policy 
$\pi^*\in\Pi$ satisfies
        \[
            \Reg_n(\pi^*) \leq \frac{\log\left( \nicefrac{2|\Pi|}{\delta} \right)}{\gamma n} + 2\gamma 
C_\gamma(\pi^*).
        \]
        Furthermore, by setting $\gamma$ to $\sqrt{\frac{\log(2\nicefrac{|\Pi|}{\delta})}{n}}$, the bound becomes
        \[
            \Reg_n(\pi^*) \leq \left(2C_\gamma(\pi^*) + 1\right) \sqrt{\frac{\log\left(\nicefrac{2|\Pi|}{\delta}
\right)}{n}}.
        \]
\end{theorem}
The bound improves on the results of \citet{Wang2023}  stated as Equation~\eqref{eq:IW} along several dimensions. 
Most importantly, our result removes the need for the behavior policy to be bounded away from zero, and as 
such completely does away with the uniform coverage assumptions needed by all previous work on the topic. 
Another improvement is that our bound tightens the dependence on the coverage ratio from $C(\pi^*)$ to the potentially 
much smaller $C_\gamma(\pi^*)$. A small practical improvement is that \alg calls the CSC oracle with a sparse input 
vector which can be computed slightly more efficiently than the dense inputs used by \citet{Wang2023}. This sparsity 
also leads to the practical advantage that \alg does not output policies that have never been observed to yield nonzero 
rewards (as long as there are alternatives that do receive positive rewards). We provide further comments on the 
tightness of the bound above and other properties of \alg in Section~\ref{sec:discussion}.\looseness=-1

The key idea behind the proof of Theorem~\ref{thm:main} is noticing that the tails of the IX estimator are asymmetric: 
since $\tV_n$ is a nonnegative random variable, its only extreme values are all going to be positive. More formally, 
this means that its lower tail will always be lighter than its upper tail, and thus a tight analysis needs to handle 
the two tails using different tools. Below, we state two lemmas that separately characterize the lower and upper tails 
of the IX estimator~\eqref{eq:IX}. The first of these bounds the upper tail along the lines of Lemma~1 (and 
Corollary~1) of \cite{Neu2015}:
\begin{lemma}\label{lem:IX_upper}
With probability at least $1-\delta$, the following holds simultaneously for all $\pi\in\Pi$:
\[
\tV_n(\pi) - v(\pi) \le \frac{\log\left(\nicefrac{|\Pi|}{\delta}\right)}{2\gamma n}.
\]
\end{lemma}
The proof is provided in Appendix~\ref{sec:proofs} for completeness, but is otherwise lifted entirely from 
\cite{Neu2015}. The second lemma provides control of the lower tail of $\tV_n$:
\begin{lemma}\label{lem:IX_lower}
With probability at least $1-\delta$, the following holds simultaneously for all $\pi\in\Pi$:
\[
v(\pi) - \tV_n(\pi) \le \frac{\log({\nicefrac{|\Pi|}{\delta}})}{2\gamma n} + 2\gamma C_\gamma(\pi).
\]
\end{lemma}
The proof follows from the observation that, since the rewards are non-negative, $\tV_n$ is a non-negative random 
variable, and as such its lower tail is well-controlled by its second moment (see, e.g., Exercise~2.9 in 
\citealp{BLM13}). The full proof is included in Section~\ref{sec:proofs} for completeness. With the above two lemmas, 
we can easily prove our main theorem.
\paragraph{Proof of Theorem~\ref{thm:main}}
The statement follows from combining the two lemmas via a union bound, and exploiting the definition of the algorithm:
\begin{align*}
 v(\hpi_n) &\ge \tV_n(\hpi_n) - \frac{\log\left(\nicefrac{2|\Pi|}{\delta}\right)}{2\gamma n} \ge \tV_n(\pi^*) - 
\frac{\log\left(\nicefrac{2|\Pi|}{\delta}\right)}{2\gamma n} \ge v(\pi^*) - 
\frac{\log\left(\nicefrac{2|\Pi|}{\delta}\right)}{\gamma n} - 2\gamma C_\gamma(\pi^*).
\end{align*}
Concretely, the first of these inequalities follows from Lemma~\ref{lem:IX_upper}, the second one from the definition 
of the algorithm, and the third one from Lemma~\ref{lem:IX_lower}. This concludes the proof.
\jmlrQED

\section{A PAC-Bayesian extension}\label{sec:PAC-Bayes}
Our previously stated results require the policy class $\Pi$ to be finite, and scale with $\log|\Pi|$. While this is a 
common assumption in past work on the subject (e.g., in \citealp{Dudik2011,Agarwal2014,Wang2023}), it is of course 
not satisfied in most practical scenarios of interest. Several extensions have been proposed in previous work, mostly 
based on the idea of replacing the union bound over policies by more sophisticated uniform-convergence arguments: for 
instance, \citet{Swaminathan2015} and \citet{Jin2023} respectively show bounds that depend on the covering number and 
the Natarajan dimension of the policy class. In this section, we provide an extension that makes use of 
so-called \emph{PAC-Bayesian} generalization bounds \citep{McA98,Aud04,Cat07} that hold for arbitrary policy classes 
and often lead to meaningful performance guarantees even in large-scale settings of practical interest. We refer to the 
recent monograph of \citet{Alq21} for a gentle introduction into the subject.

Before providing this extension, we will require some additional definitions. In this section, we will consider 
\emph{randomized} algorithms that output a distribution $\hQ_n \in \Delta_{\Pi}$ over policies, and we will be 
interested in the performance guarantees that hold on expectation with respect to the random choice of $\hpi_n \sim 
\hQ_n$, but still hold with high probablity with respect to the realization of the random data set. We overload our 
notation slightly by defining $v(Q) = \int v(\pi) \dd Q(\pi)$, $\tV_n(Q) = \int \tV_n(\pi) \dd Q(\pi)$, $C_\gamma(Q) = 
\int C_\gamma(\pi) \dd Q(\pi)$, and $\Reg_n(Q) = \int \Reg_n(\pi) \dd Q(\pi)$, which all capture relevant 
quantities evaluated on expectation under the distribution $Q\in\Delta_{\Pi}$.

In the context of offline learning, several works have applied PAC-Bayesian techniques to provide concentration bounds 
for the importance-weighted estimator $\hV_n(Q)$, characterizing its deviations from its true mean $v(Q)$ uniformly 
for all ``posteriors'' $Q$---we refer to the recent work of \citet{Sakhi2023} and the survey of \citet{Flynn2023} for 
an 
extensive overview of such results. 
One common feature of these works is that they all provide concentration bounds derived from PAC-Bayesian versions of 
standard bounds like Hoeffding's or Bernstein's inequality, and as such suffer from the same limitations as the results 
described in Section~\ref{sec:PIWO}. The biggest such limitation is that all bounds require a uniform coverage 
assumption $\inf_{x,a} \mu(a|x) \ge \alpha$, or work with biased estimates of $v(Q)$ without quantifying the effect of 
the bias on the learning performance. Instead of deriving regret bounds from the concentration bounds, the focus in 
these works is to derive implementable algorithms from the concentration bounds and test them extensively in 
large-scale settings.

Here, we provide a natural extension of \alg that is derived from PAC-Bayesian principles. For defining our algorithm, 
we let $P \in \Delta_{\Pi}$ be an arbitrary ``prior'' over the policy class $\Pi$ and define the 
output distribution as
\[
 \hQ_n = \arg\max_{Q\in\Delta_\pi} \ev{\tV_n(Q) - \frac{\KL{Q}{P}}{\lambda}}, 
\]
where $\KL{Q}{P} = \int \log \frac{\dd Q}{\dd P} \dd Q$ is the \emph{Kullback--Leibler divergence} (or \emph{relative 
entropy}) between the distributions $Q$ and $P$, and $\lambda > 0$ is a regularization parameter.
It is well known that this distribution (often called the \emph{Gibbs posterior}) has a closed-form expression with 
$\frac{\dd \hQ_n}{\dd P}(\pi) = \frac{e^{\lambda 
\tV_n(\pi)}}{\int e^{\lambda \tV_n(\pi')} \dd P(\pi')}$. For practical purposes, we will simply choose $\lambda = 
2\gamma n$ below. The following theorem establishes a regret guarantee for the resulting algorithm that we call 
\emph{PAC-Bayesian \alg}.
\begin{theorem}\label{thm:PAC-Bayes}
With probability at least $1-\delta$, the regret of PAC-Bayesian \alg against any distribution $Q^* \in \Delta_{\Pi}$ 
over comparator policies satisfies
\[
\Reg_n(Q^*) \leq \frac{\KL{Q^*}{P} + \log(1/\delta)}{\gamma n} + 2\gamma C_\gamma(Q^*).
\]
Furthermore, by setting $\gamma = \sqrt{\nicefrac{1}{n}}$, the bound becomes
        \[
            \Reg_n(Q^*) \leq \frac{2C_\gamma(Q^*) + \KL{Q^*}{P} + \log(1/\delta)}{\sqrt{n}}.
        \]
\end{theorem}
This bound inherits the key strength of PAC-Bayesian generalization bounds: it holds \emph{uniformly for all 
competitors $Q^*$} without requiring a union bound over policies. We warn the reader familiar with PAC-Bayesian 
bounds though that the role of $Q^*$ here is different from what they may expect: instead of being a data-dependent 
``posterior'', it is a ``comparator'' distribution that the learner wishes to compete with. Thus, the bound expresses 
that distributions $Q^*$ that are closer to the ``prior'' $P$ in terms of relative entropy are ``easier'' to compete 
with. As before, the bound scales with the smoothed policy coverage ratio $C_\gamma(Q^*)$, only this time associated 
with the comparator distribution $Q^*$. Just like in the bound of Theorem~\ref{thm:main}, the bound requires no uniform 
coverage condition, and in particular continues to hold even if $\inf_{x,a} \mu(a|x)$ approaches zero. To our 
knowledge, this is the first regret bound for offline learning of such a PAC-Bayesian flavor, and in any case the first 
PAC-Bayesian bound for this setting that does not require uniform coverage.

The proof of Theorem~\ref{thm:PAC-Bayes} relies on the following generalizations of 
Lemmas~\ref{lem:IX_upper} and~\ref{lem:IX_lower}:
\begin{lemma}\label{lem:IX_upper_PAC-Bayes}
With probability at least $1-\delta$, the following holds simultaneously for all $Q\in\Delta_{\Pi}$:
\[
\tV_n(Q) - v(Q) \le \frac{\KL{Q}{P}+\log\left(1/\delta\right)}{2\gamma n}.
\]
\end{lemma}
\begin{lemma}\label{lem:IX_lower_PAC-Bayes}
With probability at least $1-\delta$, the following holds simultaneously for all $Q\in\Delta_{\Pi}$:
\[
v(Q) - \tV_n(Q) \le \frac{\KL{Q}{P}+\log\left(1/\delta\right)}{2\gamma n} + 2\gamma C_\gamma(Q).
\]
\end{lemma}
The statements follow from combining the proofs of Lemmas~\ref{lem:IX_upper} and~\ref{lem:IX_lower} with a so-called 
``change-of-measure'' trick commonly used in the PAC-Bayesian literature. We relegate the proofs to 
Appendix~\ref{sec:proofs_PAC-Bayes} and only provide the very simple proof of Theorem~\ref{thm:PAC-Bayes} here.
\paragraph{Proof of Theorem~\ref{thm:PAC-Bayes}}
The statement follows from combining the above two lemmas via a union bound, and exploiting the definition of the 
algorithm:
\begin{align*}
 v(\hQ_n) &\ge \tV_n(\hQ_n) - \frac{\textrm{KL}\bigl(\hQ_n\big\|P\bigr) + \log(1/\delta)}{2\gamma n} 
 \ge \tV_n(Q^*) - 
\frac{\KL{Q^*}{P} + \log(1/\delta)}{2\gamma n} 
\\
 &\ge v(Q^*) - 
\frac{\KL{Q^*}{P} + \log(1/\delta)}{\gamma n} - 2\gamma C_\gamma(Q^*).
\end{align*}
Concretely, the first of these inequalities follows from Lemma~\ref{lem:IX_upper_PAC-Bayes}, the second one from the 
definition of the algorithm, and the third one from Lemma~\ref{lem:IX_lower_PAC-Bayes}. This concludes the proof.
\jmlrQED

    \section{Experiments}\label{sec:experiments}
    In this section we provide a set of simple experiments that illustrate our theoretical findings, and in particular 
    to empirically validate the robustness of our algorithm to hyper-parameter selection. We compare our method 
    (\alg) to the method of \citet{Wang2023} (here referred to as PIWO-PL), and follow an experimental setup 
    that is directly inspired by theirs. Besides PIWO-PL, we also include a commonly used variant of our algorithm that 
uses     the \emph{clipped importance weights} (CIW) estimator defined as $\hV_n(\pi) = \frac 1n \sum_{t=1}^n 
\frac{\pi(A_t|X_t)}{\max\ev{\mu(A_t|X_t),\gamma}} \cdot R_t$. We refer to this metod as \textsc{PIWO-Clip}.

    We use the Letter (OpenML ID 247\footnote{\url{https://www.openml.org/search?type=data\&status=active\&id=247}})
    classification dataset to simulate an offline contextual bandit instance.
    The dataset contains one milion entries, each consisting of 16 features
    and a true label, representing one of the $K=26$ letters of the alphabet.
    To simulate a contextual bandit instance, we consider the feature vectors as
    contexts and the true labels as the corresponding optimal actions.
    To simulate the rewards we build a reward matrix $M\in\Reals^{K\times K}$
    with entries on the diagonal set to $1$ and the rest of them uniformly
    sampled from the $[0,1)$ interval, and we keep these random parameters fixed for all repetitions. We 
    then set the reward distribution $p(\cdot|x,a)$ for each context-action pair $(x,a)$ as a Bernoulli distribution
    with parameter $M_{a,a*(x)}$, where $a^*(x)$ denotes the optimal action associated with context $x$.

    The cost-sensitive classification oracle is implemented by fitting a multi-variate ridge regressor, with one 
target for each action\footnote{The choice of the regularization parameter $\alpha$
    did not seem to impact significantly the result of the experiments.}.
    Given any context $x$, the regressor can be queried to predict the reward for
    each arm, and a \emph{max} or \emph{softmax} can be used to construct
    a policy to select the best arm.
    \begin{figure}[t]
        \centering
        \includegraphics[width = .95\textwidth]{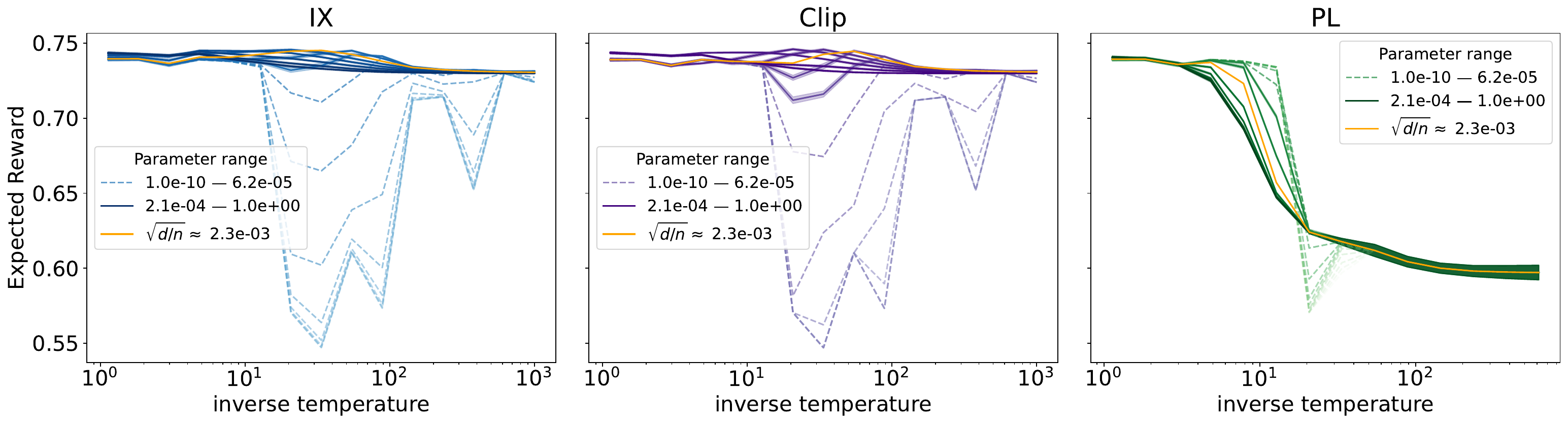}
         \caption{The performance of \alg, \textsc{PIWO-Clip}, and the algorithm of \citet{Wang2023} as a function of 
        the softmax parameter of the behavior policy. Different curves correspond to different hyperparameters for the 
        algorithms, with lighter tones representing smaller hyperparameters and darker tones representing larger ones.}
        \label{fig:experiment}
    \end{figure}
    In order to generate a range of behavior policies, we retain 10\% of the data to train an estimator
    of the reward for each arm using the regressor described above with the true mean rewards as labels. We then use 
    the predicted rewards to construct 20 softmax behavior policies, by varying the inverse temperature parameter as
    \texttt{logspace(-1, 3, 20)}.

    We then collect an offline dataset using each of the behavior policies and train our method \alg, 
    its variation \textsc{PIWO-Clip}, and the algorithm of \citet{Wang2023}, using the CSC oracle described above with 
    an argmax
    to select the optimal action, and varying their hyper-parameter over a wide range (i.e.
    \texttt{logspace(-10, 0, 20)}).
    Finally we compute the expected reward for each combination of behavior
    policy and hyper-parameter, and show the result in Figure~\ref{fig:experiment}.
    It can be observed how most choices of hyper-parameters result in good performance
    for \alg and \textsc{PIWO-Clip}, while the same cannot be said for PIWO-PL, which is very sensitive
    to small probabilities in the behavior policy and needs to compensate them
    with a very careful choice of its hyper-parameter. In particular, we note that in some experiments with 
    large softmax parameters, $\mu(a|x)$ can be as low as $10^{-100}$ for some context-action pairs, and thus 
    even a seemingly negligible regularization parameter like $\beta = 10^{-20}$ can result in massive 
    pessimistic adjustments. In contrast, \alg is robust to the presence of such 
    small observation probabilities and continues to work well for a broad range of hyperparameter choices. As 
    expected, \textsc{PIWO-Clip} performs very similarly to PIWO-IX due to the close similarity between these two 
    methods.
    More details about the experiments are provided in Appendix~\ref{sec:more-experiments}.

\section{Discussion}\label{sec:discussion}
We now provide some additional discussion on our results, related work, and open problems.

\paragraph{No more uniform coverage.} The bounds we have proved are tighter than any that are known 
in the literature, and they have the particular strength that they do not require the action probabilities to be 
strictly bounded away from zero. Virtually all previous bounds require this ``uniform-coverage'' assumption, largely 
due to their excessive reliance on textbook concentration results like Bernstein's, Bennett's, or Freedman's 
inequalities. The only result we are aware of that does not explicitly suffer from this limitation is by 
\citet{Jin2023}, 
who rely on a very sophisticated new proof technique which eventually does not yield easily interpretable performance 
bounds due to the appearance of some higher moments of the importance weights. The key to our stronger results is the 
observation that the tails of importance-weighted reward estimators are \emph{asymmetric}: their lower tails are always 
light, and thus one only has to tame the upper tails via pessimistic adjustments. This simple observation 
allows us to derive very tight bounds using a few lines of elementary derivations. If there is any moral to this story, 
then it is that one should always avoid using two-sided concentration inequalities for importance-weighted estimators 
(at least as long as the rewards are positive).

\paragraph{Implicit exploration and clipped importance weighting.} A perhaps more traditional way to control the tails 
of importance-weighted estimators is the clipped importance weighting (CIW) estimator we have defined in 
Section~\ref{sec:experiments}. Variants of this 
estimator have been studied at least since the work of \citet{Ion08} and vigorously applied in the offline 
learning literature \citep{Bottou2013,Sakhi2023,Flynn2023}. Interestingly, despite its broad usage, we are not 
aware of any work in this context that has worked out expressions for the bias of the CIW estimator, much less derived 
a 
regret bound for the resulting offline learning scheme. We believe that our results for the closely related IX 
estimator should essentially all apply to the CIW estimator and indeed our experiments show that they behave 
nearly identically in the settings we have tested. Nevertheless, we suspect that analyzing this estimator would end up 
being considerably more involved than our own analysis, but of course we would love to be proved wrong by future 
work.

\paragraph{Reward-scaled coverage ratios.} A subtle improvement of our bounds as opposed to the ones of 
\citet{Wang2023} is that they depend on the \emph{reward-scaled} version of the coverage ratio. 
This implies that bounds expressed in terms the scaled ratio $C_\gamma(\pi^*)$ can be much tighter than ones expressed 
in terms of $C(\pi^*)$ when the rewards of the comparator policy $\pi^*$ ``tend to be small'' in an appropriate sense. 
Note that this is a significant improvement in practical applications like online recommendation systems, where 
expected rewards correspond to clickthrough rates, which are very close to zero even for the very best ad campaigns. 
In the special case where rewards are negatively correlated with the importance weights (which may intuitively happen 
if the behavior policy is ``reasonably good'' in the sense that it puts larger weights on good actions), the coverage 
ratio against the optimal \emph{deterministic} policy $\pi^*$ can be shown to satisfy
$
C_\gamma(\pi^*)  \le
v(\pi^*) C(\pi^*)$,
thus improving greatly over standard bounds that depend on $C(\pi^*)$. Bounds that improve for small expected rewards 
are known in the bandit literature at least since the work of \citet{ACBFS02}, and we are curious if guarantees like 
the 
above can be proved under more general conditions for offline learning as well. \looseness=-1

\paragraph{Lower bounds.} The ``optimality'' of pessimistic offline learning methods is a contentious topic that we 
prefer not to discuss here in much detail. In particular, even in the simplest case of offline learning in multi-armed 
bandits, \citet{Xiao2021} have shown that a large range of algorithms including pessimistic, greedy, and optimistic 
methods satisfies the standard notion of minimax optimality, and there is thus nothing special about pessimistic 
methods in these terms. Putting this alarming concern aside, pessimistic algorithms tend to have the property that 
their regret scales with the minimax sample complexity of \emph{estimating} the value of the comparator policy  
\citep{Xiao2021,JYW21}. In our case, it is not entirely 
clear if this statement continues to be true. In the special case of multi-armed bandits with binary rewards and a 
deterministic comparator policy, our bound  matches the lower bound proved by 
\citet{LMS15} (up to a $\log K$ factor). That said, already in the case of stochastic comparator policies, our 
upper bounds no longer match the minimax sample complexity of estimation. Finding out if better algorithms with 
matching regret guarantees can be developed is a very interesting research question that we leave open for now.

\paragraph{Computational-statistical tradeoffs.}
As we show in this paper, it is possible to develop oracle-efficient algorithms with good statistical guarantees. 
However, these algorithms don't seem to demonstrate the correct scaling with the problem 
complexity unless prior knowledge of problem parameters is provided to the algorithm. This limitation can be bypassed 
by a more involved algorithm we describe in Appendix~\ref{sec:adaptive}, but the resulting method cannot apparently be 
implemented via a single call to the optimization oracle. Whether or not this computational-statistical tradeoff is 
inherent to the problem is unclear at this point and warrants further research.

\paragraph{Further refinements.} 
Our results can be extended in a number of straightforward ways by building on previous developments in the literature. 
For instance, the dependence on $\log |\Pi|$ appearing in our main results can be most likely replaced by other 
complexity measures like covering numbers or the Natarajan dimension of the policy class, by adapting the techniques of 
either \citet{Swaminathan2015} or \citet{Jin2023}. Similar bounds can be recovered by our PAC-Bayesian guarantees 
presented in Section~\ref{sec:PAC-Bayes} by building on techniques of \citet{Aud04,Cat07} (see also \citealp{GSZ21}). 
Another very simple generalization that our framework can readily handle is the case of adaptive behavior policies, 
where each sample point $(X_t,A_t,R_t)$ can be generated by a different behavior policy $\mu_t$ that may potentially 
depend on all past observations. The concentration bounds of Lemmas~\ref{lem:IX_upper} and~\ref{lem:IX_lower} can be 
very easily adapted to deal with such observations, and accordingly a version of our main result can be proved with the 
quantity $\EE{\frac 1n \sum_{t=1}^n \sum_a \frac{\pi^*(a|X_t)}{\mu_{t}(a|X_t) + \gamma}\cdot r(X_t,a)}$ taking the role 
of $C_\gamma(\pi^*)$. We hope that the simplicity of our techniques will enable further progress on the topic of 
importance-weighted offline learning, and in particular that further interesting extensions will be uncovered by future 
work.

    \bibliography{references.bib,shortconfs.bib}

\clearpage
\appendix

\section{Omitted proofs}\label{sec:proofs}
In this section, we prove our main technical lemmas. To facilitate this effort, we introduce the shorthand 
notations
\[
\rht = \frac{\pi(A_t|X_t)}{\mu(A_t|X_t)}R_t,
\quad\quad\mbox{and}\quad\quad
\rtt = \frac{\pi(A_t|X_t)}{\mu(A_t|X_t)+\gamma}R_t,
\]
and note that $\hV_n = \frac 1n \sum_{t=1}^n \rht$ and $\tV_n = \frac 1n \sum_{t=1}^n \rtt$, and also recall that 
$\E{\rht} = v(\pi)$ and $\E{\rtt} = v(\pi) - \gamma C_\gamma(\pi)$ holds for all $t$.

\subsection{The proof of Lemma~\ref{lem:IX_upper}}
Fix an arbitrary $\pi \in \Pi$. We start by using the elementary inequality $\log(1+y) \ge \frac{y}{1 + y/2}$ that 
holds for all $y \ge 0$ to show that
    \[
        \rtt = \frac{\pi(A_t|X_t)R_t}{\mu(A_t|X_t)+\gamma} \leq
            \frac{\pi(A_t|X_t)R_t}{\mu(A_t|X_t)+\gamma \pi(A_t|X_t)R_t} =
            \frac{1}{2\gamma}\cdot\frac{2\gamma \rht}{1 + \gamma \rht}
            \leq \frac{\log{(1 + 2\gamma \rht)}}{2\gamma}.
    \]
    This implies that
    \[
        \E{e^{2\gamma \rtt}} \leq \E{1 + 2\gamma\rht} = 1 + 2\gamma v(\pi) \leq e^{2\gamma v(\pi)},
    \]
    where the last step follows from the inequality $e^y \ge 1 + y$ that holds for all $y\in\real$.
    Using the independence of all observations, this implies $\E{e^{2\gamma\sumt (\rtt - v(\pi))}} \leq 1$,
    and thus an application of Markov's inequality yields
    \[
        \PP{\sumt (\rtt - v(\pi)) \geq \varepsilon} = \PP{e^{2\gamma\sumt (\rtt - v(\pi))} \ge e^{2\gamma\varepsilon}} 
\le 
e^{-2\gamma \varepsilon}
    \]
for any $\varepsilon \ge 0$. Setting $\varepsilon = \frac{\log(|\Pi|/\delta)}{2\gamma}$ and taking a union bound 
over all policies concludes the proof.
\jmlrQED
    
\subsection{The proof of Lemma~\ref{lem:IX_lower}}
Fix an arbitrary $\pi \in \Pi$. We start by noting that for any nonnegative random variable $Y$, and for any positive 
$\lambda$, we have
\[
 \E{e^{-\lambda Y}} \le \E{1 - \lambda Y + \lambda^2 Y^2 / 2} \le e^{- \lambda \E{Y} + \lambda^2 \E{Y^2} / 2},
\]
where the first inequality follows from $e^{-y} \le 1 - y + y^2 /2$ that holds for all $y\ge 0$ and the second from 
$e^{y} \ge 1 + y$ that holds for all $y\in \real$. Apply this inequality with $Y = \rtt$ and note that
\begin{align*}
        \E*{\pa{\rtt}^2} &= \E*{\frac{\pa{\pi(A_t|X_t)}^2}{(\mu(A_t|X_t)+\gamma)^2}\cdot R^2_t}
        \le 
        \E*{\sum_a \II{A_t = a} \frac{\pi(a|X_t) }{(\mu(a|X_t)+\gamma)^2}\cdot r(X_t,a)} 
        \\
        & \leq \E*{\sum_a \frac{\pi(a|X_t) }{\mu(a|X_t)+\gamma} \cdot r(X_t,a)}= C_\gamma(\pi),
    \end{align*}
    where the first inequality used the boundedness of the rewards to show $\EE{R_t^2} \le \EE{R_t} = \EE{r(X_t,a)}$ 
and $\pa{\pi(a|X_t)}^2 \le \pi(a|X_t)$, and the second inequality used that $\EEcc{\II{A_t=a}}{X_t} = \mu(a|X_t)$.

Using the independence of all observations, this implies $\E{e^{\lambda\sumt (\E{\rtt} - \rtt - \lambda 
C_\gamma(\pi)/2)}} \leq 1$. Recalling that $\E{\rtt} = v(\pi) - \gamma C_\gamma(\pi)$, an application of Markov's 
inequality yields
    \[
        \PP{\sumt \pa{v(\pi) - \rtt - \pa{\gamma + \lambda /2} C_\gamma(\pi)} \geq \varepsilon} \le 
e^{-2\gamma\varepsilon}.
    \]
Setting $\lambda = 2\gamma$ and $\varepsilon = \frac{\log(|\Pi|/\delta)}{2\gamma}$, and finally taking a union bound 
over all policies concludes the proof.
\jmlrQED

\subsection{The proofs of Lemmas~\ref{lem:IX_upper_PAC-Bayes} 
and~\ref{lem:IX_lower_PAC-Bayes}}\label{sec:proofs_PAC-Bayes}
To prove Lemma~\ref{lem:IX_upper_PAC-Bayes}, let us first fix an arbitrary $Q \in \Delta_{\Pi}$, and recall from 
the proof of Lemma~\ref{lem:IX_upper} that $\E{e^{2\gamma \rtt}} \leq e^{2\gamma v(\pi)}$, holds for all fixed $\pi$. 
Thus, since $P$ is independent of the random observations, we also have
    \[
        \EE{\int e^{2\gamma \sum_{t=1}^n \pa{\rtt - v(\pi)}} \dd P(\pi)} \leq 1.
    \]
    Now, let us introduce the notation $\rho_\pi(Q,P) = \log \frac{\dd Q}{\dd P} (\pi)$ and write
    \begin{align*}
       & \PP{\int \pa{\sumt (\rtt - v(\pi)) - \frac{\rho_\pi(Q,P)}{2\gamma}}\dd Q(\pi) \geq 
\varepsilon} 
\\
&\qquad \le
        \EE{e^{2\gamma \int \pa{\sumt (\rtt - v(\pi))  - \frac{\rho_\pi(Q,P)}{2\gamma}}\dd Q(\pi)}}
e^{-2\gamma \varepsilon}
\\
&\qquad \le \EE{\int e^{2\gamma \pa{\sumt (\rtt - v(\pi))  - \frac{\rho_\pi(Q,P)}{2\gamma}}}\dd Q(\pi)}
e^{-2\gamma\varepsilon} 
\\
&\qquad = \EE{\int e^{2\gamma \pa{\sumt (\rtt - v(\pi))}} \frac{\dd P}{\dd Q}(\pi) \dd Q(\pi)}
e^{-2\gamma \varepsilon}
\\
&\qquad = \EE{\int e^{2\gamma \pa{\sumt (\rtt - v(\pi))}} \dd P(\pi)}e^{-2\gamma \varepsilon} \le 
e^{-2\gamma \varepsilon}.
    \end{align*}
    Here, the first step follows from Markov's inequality, the second from Jensen's inequality for the convex function 
$y\mapsto e^{2\gamma y}$, the third from the definition of $\rho_\pi(Q,P)$, the fourth from the definition of the 
Radon--Nykodim 
derivative $\frac{\dd P}{\dd Q}$, and the last step from the inequality that we have established above.
Noticing that $\int \rho_\pi(Q,P) \dd Q(\pi) = \KL{Q}{P}$ and setting $\varepsilon = \frac{\log(1/\delta)}{2\gamma}$ 
concludes the proof of Lemma~\ref{lem:IX_upper_PAC-Bayes}. The proof of Lemma~\ref{lem:IX_lower_PAC-Bayes} then follows 
analogously by recalling from the proof of Lemma~\ref{lem:IX_lower} that $\E{e^{2\gamma(v(\pi) - \tr_t(\pi) - 2\gamma 
C_\gamma(\pi))}} \leq 1$, and then following the same steps as above.
\jmlrQED

\section{Adaptivity to the coverage}\label{sec:adaptive}
One shortcoming of the result in Theorem~\ref{thm:main} is that it scales linearly with $C_\gamma(\pi^*)$ even though 
prior results suggest that a scaling with $\sqrt{C_0(\pi^*)}$ should be possible \citep{Swaminathan2015,Wang2023}. 
This improvement can be trivially achieved by setting $\gamma = \sqrt{\frac{\log(|\Pi|/\delta)}{C_0(\pi^*)n}}$, but 
this requires prior knowledge of $C_0(\pi^*)$ which is of course unavailable in practice (at least in the most 
interesting case where $\pi^*$ is the optimal policy).

This limitation can be addressed by defining the following \emph{non-uniformly scaled} version of the IX estimator:
\begin{equation}\label{eq:IX-scaled}
 \tV^\dag_n(\pi) = \frac 1n \sum_{t=1}^n \frac{\pi(A_t|X_t)}{\mu(A_t|X_t) + \gamma_\pi} \cdot R_t - \frac{\log \pa{
|\Pi|/\delta}}{2\gamma_\pi}.
\end{equation}
Here,  $\gamma_\pi > 0$ is a \emph{policy-dependent} IX parameter that is potentially different for each policy $\pi$.
Using this estimator, we define a variant of our main algorithm called \emph{coverage-scaled PIWO-IX} that outputs
\[
 \hpi_n = \argmin_{\pi\in\Pi} \tV^\dag_n(\pi).
\]
Notice that, unlike PIWO-IX, this algorithm cannot be directly implemented 
using a standard optimization oracle due to the policy-dependent IX parameters $\gamma_\pi$.
The following theorem is straightforward to prove using our previously established Lemmas~\ref{lem:IX_upper} 
and~\ref{lem:IX_lower}:
\begin{theorem}\label{thm:adaptive}
With probability at least $1-\delta$, the regret of coverage-scaled \alg against any comparator policy $\pi^*\in\Pi$ 
satisfies
        \[
            \Reg_n(\pi^*) \leq \frac{\log\left( \nicefrac{2|\Pi|}{\delta} \right)}{\gamma_{\pi^*} n} + 
2\gamma_{\pi^*} 
C_{\gamma_{\pi^*}}(\pi^*).
        \]
        Furthermore, by setting $\gamma_\pi = \sqrt{\frac{\log(\nicefrac{2|\Pi|}{\delta})}{2C_0(\pi)n}}$ for 
        each $\pi$, the bound becomes
        \[
            \Reg_n(\pi^*) \leq  \sqrt{\frac{8C_0(\pi^*)\log\left(\nicefrac{2|\Pi|}{\delta}
\right)}{n}}.
        \]
\end{theorem}
\begin{proof}
First observe that the statements of Lemmas~\ref{lem:IX_upper} and~\ref{lem:IX_lower} can be trivially adjusted to show 
that the bounds
\[
0 \le v(\pi) - \tV_n^\dag(\pi) \le \frac{\log({\nicefrac{2|\Pi|}{\delta}})}{\gamma_\pi n} + 2\gamma C_{\gamma_\pi}(\pi).
\]
hold simultaneously for all policies with probability at least $1-\delta$.
Then, by the definition of the algorithm, we obtain
\begin{align*}
 v(\hpi_n) &\ge \tV^\dag_n(\hpi_n) \ge \tV^\dag_n(\pi^*) \ge v(\pi^*) - 
\frac{\log\left(\nicefrac{2|\Pi|}{\delta}\right)}{\gamma_{\pi^*} n} - 2\gamma_{\pi^*} C_{\gamma_{\pi^*}}(\pi^*).
\end{align*}
This concludes the proof of the first claim. The second claim can be verified by noticing that $C_\gamma(\pi^*) \le 
C_{0}(\pi^*)$ for all $\gamma > 0$ and plugging in the choice of $\gamma_{\pi}$ stated in the theorem.
\end{proof}

\section{Further details on the experiments}\label{sec:more-experiments}
    In this section we give more detail on all the experiments we ran. 
    The first step we performed was to use $10\%$ of the data to fit a multivariate
    ridge regressor $\texttt{reg}(x,a)$ to predict the expected reward of each action,
    given any context. For each context $x$ and each corresponding optimal action $a^*$
    in the data, we selected $M_{\cdot,a^*}$
    as the label vector (having one entry for each possible action).

    We then used the reamining $90\%$ of the data to perform two sets of
    experiments. In the
    first set, which is the one described in the main text (Section~\ref{sec:experiments}),
    we considered 20 softmax behavior policies, varying their inverse
    temperature parameter $\eta$ as \texttt{logspace(-1, 3, 20)}. That is,
    \[
        \pi_\eta(a|x) \propto \exp(\eta\, \texttt{reg}(x,a)).
    \]

    We repeated each set of experiments 10 times ($i\in[10]$), using a $10$-fold validation procedure.
    That is, the data was first partitioned into 10 non overlapping folds.
    On each repetition $i$, 9 folds are used to generate the training data for the
    algorithms, by simulating the interaction of each behavior policy $\pi_\eta$ and the
    bandit instance. The resulting training dataset $\D_{\eta,i}$ was used to train
    each algorithm for each possible hyper-parameter choice $h\in\texttt{logspace(-10, 0, 20))}$.
    Finally, each trained algorithm $\mathfrak{A}_{\eta,i,h}$ is evaluated using the
    data in the remaining fold, by computing the expected regret using the
    true mean rewards.

    This set of experiments was then repeated for a different set of ``bad''
    behavior policies, which were defined as
    \[
        \pi_\eta(a|x) \propto \exp(-\eta\, \texttt{reg}(x,a)).
    \]
    
    The results for the two sets of experiments are shown respectively in
    Figures~\ref{fig:full-good} and~\ref{fig:full-bad}.
    On each figure, the first row of plots shows the expected reward as a function
    of the inverse temperature parameter $\eta$. Each plot on the row is for
    one of the three different algorithms, and it contains a line for each
    possible hyper-parameter. The lines are colored using a gradient from
    lighter to darker to represent increasing hyper-parameter values. In orange
    we highlighted the learning rate corresponding to $\sqrt{d/n}$, which we use as a 
    crude approximation of the hyper-parameter recommended by
    theory, $\sqrt{\log|\Pi|/n}$. In addition,
    values of the hyper-parameters much smaller than $\sqrt{d/n}$ are represented
    with a dashed line. All lines (excluding for clarity of the representation
    the dashed ones) have a shaded region representing the standard
    deviation over the 10 runs. The second row of plots shows the expected
    regret as a function of the hyper-parameter $h$. Thus, we can observe
    a line for each different behavior policy parameter $\eta$. Here the lines
    are lighter for smaller values of $\eta$, and darker for bigger values
    of $\eta$.
    
From the plots, we can infer that PIWO-IX performs well when the behavior policy is ``good'' and 
$\gamma$ is set in a broad proximity of its theoretically recommended value. This behavior appears to be robust as we 
vary the degree of ``goodness'' of the policy modulated by the softmax parameter $\eta$, and in particular performance 
stays good even as $\eta$ approaches its higher extremes and the behavior policy gets more and more deterministic. As 
expected, \textsc{PIWO-Clip} behaves comparably. In comparison PIWO-PL is a lot less robust in this case and its 
performance decays as $\eta$ increases, most likely due to the more and more extreme values of the importance weights 
arising from some sampling probabilities approaching zero. We note that the the case of ``good'' behavior policies is 
the most practical use case, and our experiments suggest that our algorithm performs excellently in this scenario for a 
wide range of hyperparameters.

In comparison, the picture changes when considering the case of ``bad'' behavior policies. In this case, \alg and 
performs worse and worse as $\gamma$ is increased, especially for large values of $\eta$ corresponding to particularly 
bad behavior policies. This is not surprising given that the policy coverage ratio blows up in this extreme, as less 
and less mass is put on well-performing actions. Also notice that increasing the regularization parameter $\gamma$ 
forces the algorithm to be more and more pessimistic and thus stay closer and closer to the behavior policy, which 
again results in decaying performance. The performance of PIWO-PL is less consistent in this case, and it is hard to 
read out patterns that are well-predicted by theory. 

    \begin{figure}[t]
        \centering
        \includegraphics[width = .95\textwidth]{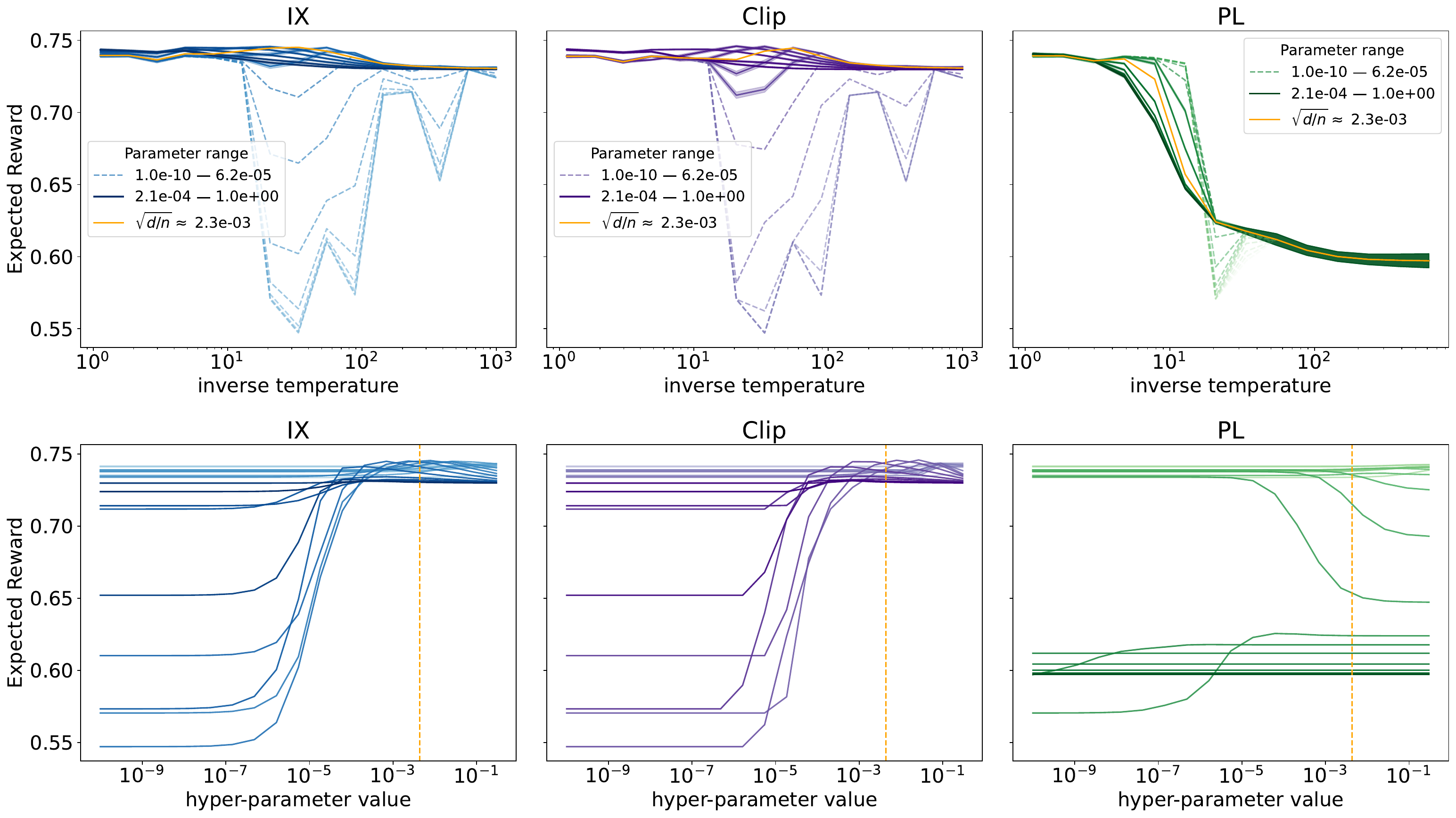}
        \caption{Results of \alg, \textsc{PIWO-Clip}, and \textsc{PIWO-PL} with good behavior policies.}
        \label{fig:full-good}
    \end{figure}
    \begin{figure}[t]
        \centering
        \includegraphics[width = .95\textwidth]{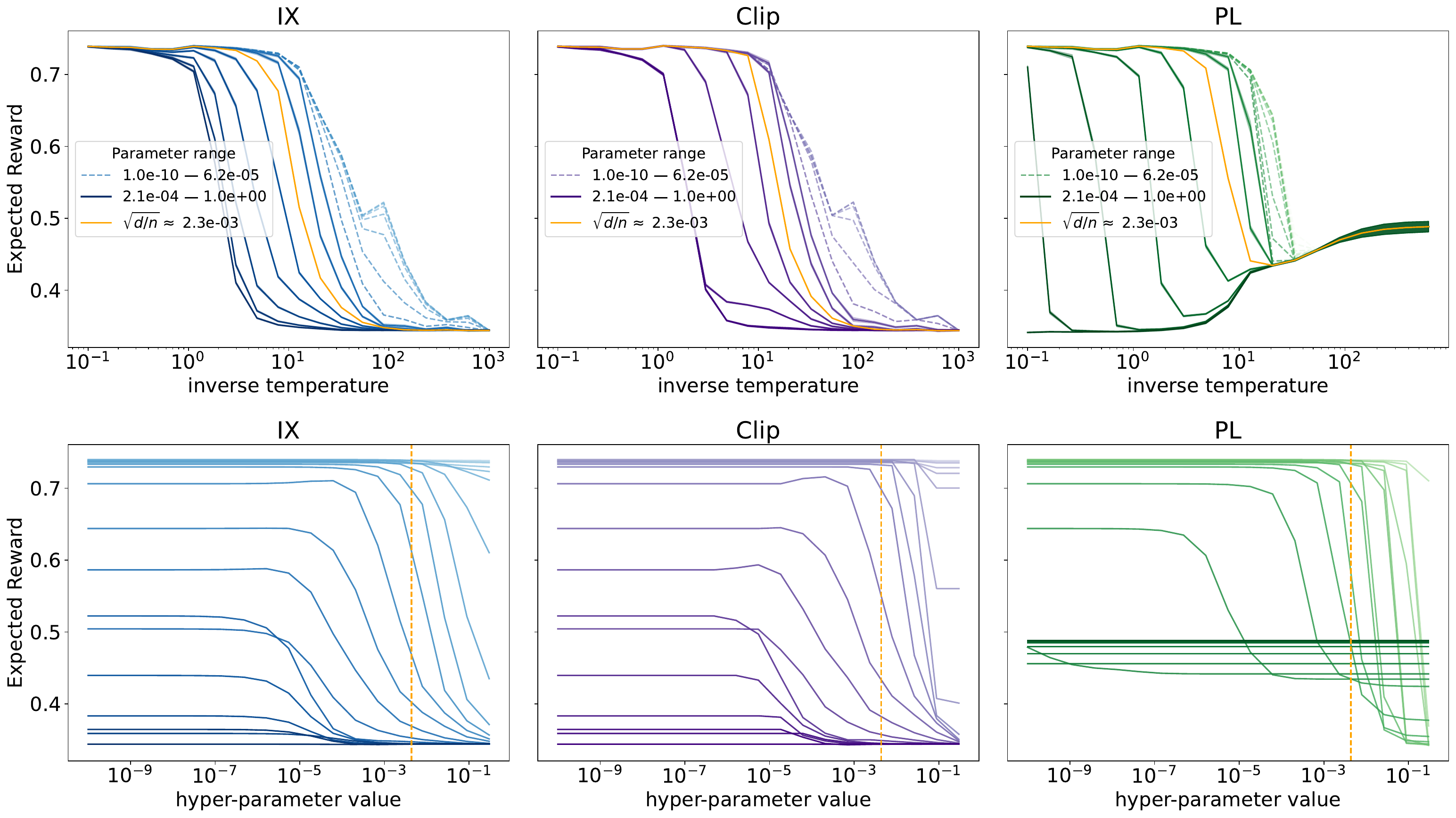}
        \caption{Results of \alg, \textsc{PIWO-Clip}, and \textsc{PIWO-PL} with bad behavior policies.}
        \label{fig:full-bad}
    \end{figure}

\end{document}